\PassOptionsToPackage{table}{xcolor}
\documentclass[pdflatex,sn-mathphys-num]{sn-jnl}

\usepackage{graphicx}
\usepackage{amsmath,amssymb,amsfonts}
\usepackage{mathtools}
\usepackage{bm}
\usepackage{booktabs}
\usepackage{multirow}
\usepackage{array}
\usepackage{tabularx}
\usepackage{placeins}
\usepackage{xcolor}

\theoremstyle{thmstyleone}
\newtheorem{theorem}{Theorem}

\newtheorem{proposition}[theorem]{Proposition}
\newtheorem{corollary}[theorem]{Corollary}

\theoremstyle{thmstyletwo}
\newtheorem{remark}{Remark}

\theoremstyle{thmstylethree}
\newtheorem{definition}{Definition}

\newcolumntype{Y}{>{\centering\arraybackslash}X}
\newcolumntype{L}{>{\raggedright\arraybackslash}X}
\newcommand{\best}[1]{\textbf{#1}}
\newcommand{\second}[1]{\underline{#1}}
\newcommand{\pc}{p_c}
\newcommand{\pa}{p_\alpha}
\newcommand{\upar}{\ensuremath{\uparrow}}
\newcommand{\downar}{\ensuremath{\downarrow}}

\definecolor{bestgreen}{HTML}{E8F3EC}
\definecolor{warningred}{HTML}{F9E9E8}
\definecolor{softblue}{HTML}{EAF2F8}

\begin{document}

\title[Control Allocation in Neural Network Optimization]{Control Allocation in Neural Network Optimization: Joint Affine Control of Weight and Bias Updates}

\author[1]{\fnm{Zhang} \sur{Gongyue}}
\author[1]{\fnm{Sheng} \sur{Yixuan}}
\author[1]{\fnm{Wang} \sur{Zhiyong}}
\author[1]{\fnm{Liu} \sur{Donghan}}
\author[1]{\fnm{Ren} \sur{Weihong}}
\author*[1]{\fnm{Liu} \sur{Honghai}}\email{honghai.liu@icloud.com}

\affil[1]{\orgdiv{State Key Laboratory of Robotics and Systems}, \orgname{Harbin Institute of Technology Shenzhen}, \orgaddress{\city{Shenzhen}, \postcode{518055}, \country{China}}}

\abstract{Optimization algorithms determine not only the magnitude of a neural-network update but also how that update is distributed across parameter channels. We study whether this distribution can be treated as a controllable quantity independently of global training progress. We define operational update allocation through normalized channel energies and analyze two scalar controls: a coordinate-preconditioning exponent and an affine spectral exponent that scales the bias column of an augmented weight--bias matrix. At a frozen state, a common nonzero step-size multiplier leaves normalized allocation unchanged; the coordinate exponent yields affine pairwise log-odds with an explicit inverse; and the affine exponent induces a rank-one positive-semidefinite Gram perturbation and a logistic raw-participation law. We further separate raw affine participation, spectral gain, and the decoded physical bias update, and show that finite polynomial spectral iterations preserve singular subspaces. Same-state replay verifies the exact control laws. On a five-seed controlled benchmark, intermediate controls improve held-out and worst-group metrics, whereas excessive affine control causes underfitting. A four-task single-seed transfer study provides descriptive corroboration. These results establish instantaneous allocation control and a bounded empirical operating regime, but do not imply a task-independent generalization ordering.}


\maketitle
\section{Introduction}
\label{sec:introduction}

A learning objective specifies which parameter values have low loss, but it
does not in general specify which low-loss solution an optimization algorithm
will select or how the training signal will be distributed on the way there.
This distinction is central in overparameterized models.  Under explicit
assumptions, gradient-based methods can select max-margin or norm-related
solutions even when the corresponding preference is absent from the stated
objective \cite{soudry2018implicit,gunasekar2018geometry,gunasekar2018conv}.
Related results establish directional convergence for broad classes of deep
homogeneous models \cite{ji2020directional}.  These works make precise an
important endpoint view of implicit bias: the optimization dynamics can
select a particular limiting direction or geometry among many candidates.

Modern optimizers also modify the geometry of individual updates.  Diagonal
adaptive methods rescale coordinates using accumulated gradient statistics
\cite{duchi2011adagrad,kingma2015adam}, while matrix methods exploit tensor or
layer structure \cite{gupta2018shampoo}.  Such transformations are not merely
global changes of step size.  They alter relative update magnitudes across
coordinates or singular directions and can therefore lead to different
solutions and different empirical behavior
\cite{wilson2017marginal,balles2018dissecting}.  Partial adaptivity
\cite{chen2020padam} and square-root-free variants
\cite{lin2024squareroot} further show that the strength and form of
preconditioning are meaningful design choices rather than fixed conventions.

Endpoint characterizations, however, do not by themselves describe how a
finite-time optimizer routes its update among specified parameter channels at
a given state.  The information-allocation view introduced by
\cite{zhang2026anatomy} addresses this training-time question through
diagnostics of gradient demand, update injection, and channel persistence.
That framework is observational: it can reveal that two training procedures
write their updates differently, but it does not yet identify when an
optimizer input is an allocation control, what allocations are reachable, or
which conclusions remain valid when the training state is allowed to change.
The joint affine construction of \cite{zhang2026joint} provides a concrete
spectral intervention by appending the bias momentum to the weight-momentum
matrix.  It leaves open the more general control question: how should the
strength of coordinate and affine participation be parameterized, and what can
be proved exactly before any task-dependent performance claim is made?

This paper formulates that question as an instantaneous control problem.  For
fixed linear channel maps $C_1,\ldots,C_K$ and a nonzero realized update
$\Delta\theta$, we define the operational allocation
\begin{equation}
    a_k(\Delta\theta)
    :=
    \frac{\lVert C_k\Delta\theta\rVert_2^2}
    {\sum_{j=1}^{K}\lVert C_j\Delta\theta\rVert_2^2}.
    \label{eq:intro_allocation}
\end{equation}
The term \emph{information} refers only to this prescribed distribution of
update energy; it does not invoke Shannon information.  We freeze the current
parameters, optimizer state, and stochastic input, and then study the map from
an optimizer control to $a(\Delta\theta)$.  This frozen-state convention is
deliberate.  It separates direct actuation by the optimizer from the indirect
effect of visiting a different state later in training.

Two controls instantiate the framework.  The coordinate exponent $p_c$ acts
through
\begin{equation}
    \Delta\theta_i(p_c)
    =-\eta\frac{m_i}{(v_i+\varepsilon)^{p_c}},
    \label{eq:intro_coordinate_rule}
\end{equation}
whereas the affine exponent $p_\alpha$ acts through the augmented matrix
\begin{equation}
    A(p_\alpha)
    =\bigl[M_W,\,d^{p_\alpha}m_b\bigr].
    \label{eq:intro_affine_rule}
\end{equation}
Here $m_i$ and $v_i$ are frozen coordinate moments, $M_W$ and $m_b$ are the
weight and bias momentum terms of an affine layer, and $d$ is the chosen layer
scale.  These parameterizations yield exact input--allocation laws.  For
active coordinates $i$ and $j$,
\begin{equation}
    \log\frac{a_i(p_c)}{a_j(p_c)}
    =
    \log\frac{m_i^2}{m_j^2}
    +2p_c\log\frac{v_j+\varepsilon}{v_i+\varepsilon},
    \label{eq:intro_coordinate_logodds}
\end{equation}
so the pairwise allocation log-odds are affine in $p_c$ and are exactly
invertible whenever $v_i\neq v_j$.  Likewise,
\begin{equation}
    A(p_\alpha)A(p_\alpha)^\top
    =M_WM_W^\top+d^{2p_\alpha}m_bm_b^\top,
    \label{eq:intro_rankone}
\end{equation}
which exposes affine participation as a rank-one positive-semidefinite
perturbation.  When both component energies are nonzero, the raw bias-column
share obeys a logistic law whose logit has slope $2\log d$ with respect to
$p_\alpha$; the slope is nonzero when $d\neq1$.  In contrast, when a single
nonzero learning-rate multiplier scales the entire update, it changes total
update magnitude but cancels from the normalized allocation in
\eqref{eq:intro_allocation} at the frozen state.

The empirical study follows the same separation of claims.  Same-state replay
tests the one-step laws with parameters, mini-batches, and optimizer moments
held fixed.  A controlled benchmark with known stable, spurious, and noise
channels then evaluates full training trajectories over five paired random
seeds.  Finally, a four-task transfer suite is reported as descriptive
corroboration because only one seed is available there.  The controlled sweeps
show a bounded useful regime and a clear over-control regime: intermediate
values can improve held-out and worst-group metrics, whereas sufficiently
large affine exponents cause training collapse.  This non-monotonicity is not
an exception to the theory.  The theory specifies how allocation responds to
the controls; it does not assert that more allocation pressure must improve a
task metric.

The contributions are as follows.
\begin{itemize}
    \item We define normalized update allocation using fixed channel maps and
    distinguish global progress control, frozen-state allocation actuation,
    and multi-step trajectory behavior.

    \item For coordinate preconditioning, we derive an exact affine law for
    pairwise allocation log-odds, an explicit inverse on the reachable set,
    and a distribution-level monotonicity identity.

    \item For affine spectral control, we derive the rank-one Gram
    perturbation, the raw participation logit, conditional attenuation laws,
    and modal-invariance statements for finite polynomial spectral maps.  The
    analysis also separates raw affine participation, preconditioner gain,
    and the decoded physical bias update.

    \item We connect the frozen-state results to trajectory-level evidence by
    combining same-state replay, complete control sweeps, five-seed paired
    comparisons, mechanism diagnostics, and explicitly limited single-seed
    transfer results.
\end{itemize}

The scope of the paper is intentionally narrower than a universal
generalization claim.  Instantaneous reachability is not controllability of
the full stochastic training recursion; a larger raw bias share need not
produce a larger physical bias update after spectral shaping and decoding; and
no value of $p_c$ or $p_\alpha$ is claimed to dominate across all tasks.  These
boundaries are part of the result: they identify what is controlled exactly,
what is only conditionally characterized, and what must be established
empirically.

\section{Related Work}
\label{sec:related_work}

\subsection{Implicit bias and optimization geometry}
\label{subsec:rw_implicit_bias}

Implicit-bias theory asks which solution is selected by an optimization
procedure when the training objective alone does not identify a unique one.
For unregularized logistic regression on linearly separable data, gradient
descent converges in direction to the hard-margin solution under the
conditions analyzed by \cite{soudry2018implicit}.  The selected solution can
depend on the geometry of the optimization method
\cite{gunasekar2018geometry} and on the parameterization: for example, linear
convolutional and fully connected parameterizations can induce different
predictor-space biases despite representing the same linear function class
\cite{gunasekar2018conv}.  Directional convergence and gradient alignment have
also been studied for deep homogeneous networks under explicit structural and
regularity assumptions \cite{ji2020directional}.

The present work is complementary to these asymptotic and solution-selection
results.  It does not infer a limiting norm or margin, and it does not require
the iterates to converge.  Instead, it conditions on a realized training state
and asks how an optimizer input changes the normalized distribution of the
next update.  This local statement can be tested by replay at the same state,
but by itself it does not determine the endpoint of the resulting trajectory.

\subsection{Coordinate-adaptive optimization and partial adaptivity}
\label{subsec:rw_coordinate_adaptivity}

AdaGrad constructs coordinate-dependent steps from accumulated gradient
information and gives data-dependent regret guarantees in online and
stochastic settings \cite{duchi2011adagrad}.  Adam combines moving averages of
first and second moments \cite{kingma2015adam}; subsequent work exhibited a
convex setting in which the original method can fail to converge and proposed
AMSGrad as a remedy \cite{reddi2018convergence}.  The behavior of adaptive
methods is not exhausted by their convergence rates.  Wilson et al.
constructed overparameterized examples in which adaptive and non-adaptive
methods select different solutions and reported empirical generalization
differences on several deep-learning tasks \cite{wilson2017marginal}.
Balles and Hennig separated sign and variance-related components of Adam to
study their distinct effects \cite{balles2018dissecting}.

Several methods expose the strength or algebraic form of adaptivity as a
design choice.  Padam introduces a partial-adaptivity exponent that connects
SGD-like and Adam/AMSGrad-like updates and analyzes convergence to stationary
points in stochastic nonconvex optimization \cite{chen2020padam}.  Lin et al.
study square-root-free adaptive methods from a second-order perspective and
show that changing the matrix or diagonal root can materially change
optimization behavior \cite{lin2024squareroot}.  Our coordinate control is
closest algebraically to partial adaptivity, but the object of analysis is
different: we derive the exact frozen-state response of normalized channel
energy, including pairwise reachability and degeneracy conditions.  These
identities neither replace convergence analyses nor imply Padam-style
generalization conclusions.

\subsection{Matrix and spectral optimization}
\label{subsec:rw_spectral_optimization}

Matrix-aware optimizers use structure that is lost when a parameter tensor is
treated as a flat collection of unrelated coordinates.  Shampoo maintains
mode-wise preconditioners and applies matrix fractional powers to tensor
gradients \cite{gupta2018shampoo}.  Modular duality interprets optimization as
mapping gradients from dual to primal spaces and derives layerwise maps,
including Newton--Schulz implementations for matrix-valued layers
\cite{bernstein2024modular}.  Muon applies matrix orthogonalization to
momentum updates, and large-scale experiments have studied how this update can
be scaled to language-model training \cite{liu2025muon}.

These methods motivate a spectral view of update allocation: a matrix map can
redistribute magnitude among singular modes even when a scalar learning rate
cannot.  Most such constructions are naturally stated for matrix-valued
weights.  An affine layer, however, contains both a matrix $W$ and a vector
$b$.  Treating $b$ with a separate vector optimizer prevents it from
participating in the same singular geometry.  Joint affine spectral shaping
addresses this mismatch by applying a spectral map to an augmented
weight--bias momentum matrix \cite{zhang2026joint}.  The present analysis
extends that construction from a fixed intervention to a continuous control
$p_\alpha$, and distinguishes three quantities that can otherwise be
conflated: the raw energy of the appended column, its gain under the spectral
map, and the final physical bias update after decoding and clipping.

\subsection{Position relative to the two antecedent studies}
\label{subsec:rw_prior_position}

Because this paper builds directly on two earlier studies, the incremental
boundary is stated explicitly.
\begin{description}
    \item[Information-allocation diagnostics.]
    \cite{zhang2026anatomy} introduces the training-time allocation viewpoint,
    observable channel diagnostics, and collapse--persistence comparisons.
    The present paper does not re-claim those diagnostics.  It turns the
    viewpoint into a controlled map, proves exact input--allocation relations,
    identifies reachable and degenerate cases, and separates one-step
    actuation from trajectory-level effects.

    \item[Joint affine spectral shaping.]
    \cite{zhang2026joint} introduces the augmented affine momentum matrix and
    evaluates a particular joint regularized-inverse optimizer.  The present
    paper does not present augmentation alone as a new contribution.  It uses
    a continuous affine exponent, derives the associated rank-one and logistic
    laws, analyzes conditional spectral attenuation and finite polynomial
    maps, and tests both useful and collapse regimes with same-state replay and
    complete sweeps.
\end{description}

The resulting contribution is therefore a control-theoretic formalization and
validation of allocation mechanisms, not a renaming of the earlier diagnostic
framework or a second report of the same fixed affine optimizer.  This
positioning also determines the permissible claims: theorems concern
frozen-state allocation maps; full-run experiments concern observed
trajectories; and task-independent performance ordering is neither assumed nor
asserted.

\section{Information Allocation as an Instantaneous Control Problem}
\label{sec:allocation_theory}

This section formalizes the allocation viewpoint introduced in
\cite{zhang2026anatomy} and the joint affine construction of
\cite{zhang2026joint}.  The word \emph{information} is used operationally:
it refers to how a realized parameter update is distributed among specified
parameter channels.  No Shannon-information interpretation is assumed.

\subsection{Controlled dynamics and operational allocation}
\label{subsec:controlled_dynamics}

Let the training state at step $t$ be $s_t=(\theta_t,\zeta_t)$, where
$\theta_t\in\mathbb{R}^{P}$ denotes the parameters and $\zeta_t$ contains all
optimizer state variables.  Let $\xi_t$ denote the stochastic input at that
step, including the sampled mini-batch, and let $u_t$ collect the user-selected
controls.  Training is represented abstractly by
\begin{equation}
    s_{t+1}=F_t(s_t,\xi_t,u_t),
    \qquad
    \Delta\theta_t
    :=\theta_{t+1}-\theta_t
    =Q_t(s_t,\xi_t,u_t).
    \label{eq:controlled_training}
\end{equation}
This representation is only a bookkeeping device; no differentiability or
linearity of $F_t$ is assumed unless stated explicitly.

Fix linear channel maps
\begin{equation}
    C_k:\mathbb{R}^{P}\longrightarrow\mathcal{H}_k,
    \qquad k=1,\ldots,K,
    \label{eq:channel_maps}
\end{equation}
where every $\mathcal{H}_k$ is a finite-dimensional Euclidean space.  For a
nonzero update $\Delta\theta$ satisfying
\begin{equation}
    D(\Delta\theta)
    :=\sum_{j=1}^{K}\lVert C_j\Delta\theta\rVert_2^2>0,
    \label{eq:allocation_denominator}
\end{equation}
define the channel energy and normalized allocation by
\begin{equation}
    E_k(\Delta\theta):=\lVert C_k\Delta\theta\rVert_2^2,
    \qquad
    a_k(\Delta\theta)
    :=\frac{E_k(\Delta\theta)}{D(\Delta\theta)}.
    \label{eq:operational_allocation}
\end{equation}
Then $a_k\geq 0$ and $\sum_k a_k=1$.  If the $C_k$ are mutually orthogonal
projectors satisfying $\sum_k C_k=I$, the $E_k$ form an orthogonal energy
partition of $\Delta\theta$.  For overlapping or non-projective channel maps,
\eqref{eq:operational_allocation} remains a normalized diagnostic, but it is
not an orthogonal conservation law.

For fixed $(t,s,\xi)$, define the frozen-state allocation response
\begin{equation}
    \Psi_{t,s,\xi}(u)
    :=a\bigl(Q_t(s,\xi,u)\bigr),
    \label{eq:frozen_response}
\end{equation}
whenever the denominator in \eqref{eq:allocation_denominator} is positive.

\begin{definition}[Instantaneous actuation]
\label{def:instantaneous_actuation}
Let $y(u)$ be a scalar component or scalar coordinate of
$\Psi_{t,s,\xi}(u)$.  A scalar control $v$ \emph{locally actuates} $y$ at
$u_0$ if $y$ is differentiable there and
$\partial y/\partial v\vert_{u_0}\neq 0$.  It \emph{exactly reaches} a set
$\mathcal{Y}$ on a control domain $\mathcal{V}$ if
$y(\mathcal{V})=\mathcal{Y}$; it is \emph{exactly invertible} if the restricted
map $y:\mathcal{V}\to\mathcal{Y}$ is one-to-one as well.
\end{definition}

Definition~\ref{def:instantaneous_actuation} concerns the one-step map with
the state and stochastic input frozen.  It is distinct from controllability of
the nonlinear, stochastic, multi-step state recursion in
\eqref{eq:controlled_training}.

\subsection{Global step size controls progress but not normalized allocation}
\label{subsec:global_scale}

\begin{theorem}[Invariance under nonzero global scaling]
\label{thm:global_scale_invariance}
Fix $(t,s,\xi)$ and all controls except a scalar step size $\eta$.  Suppose
that the update has the form
\begin{equation}
    Q_t(s,\xi,\eta)=-\eta q,
    \label{eq:global_scaled_update}
\end{equation}
where $q\in\mathbb{R}^{P}$ is independent of $\eta$ and
$D(q)>0$.  Then, for every $\eta\neq 0$,
\begin{equation}
    a_k(-\eta q)=a_k(q),
    \qquad k=1,\ldots,K,
    \label{eq:scale_invariance}
\end{equation}
whereas
\begin{equation}
    \lVert Q_t(s,\xi,\eta)\rVert_2
    =|\eta|\lVert q\rVert_2.
    \label{eq:progress_scaling}
\end{equation}
Thus a nonzero global step size changes update magnitude but does not actuate
the normalized instantaneous allocation.
\end{theorem}

\begin{proof}
Linearity of every $C_k$ gives
\begin{equation}
    E_k(-\eta q)
    =\lVert -\eta C_kq\rVert_2^2
    =\eta^2 E_k(q).
\end{equation}
Consequently $D(-\eta q)=\eta^2D(q)>0$ for $\eta\neq0$, and the common factor
$\eta^2$ cancels from the ratio in
\eqref{eq:operational_allocation}.  Equation
\eqref{eq:progress_scaling} follows from absolute homogeneity of the Euclidean
norm.
\end{proof}

\begin{remark}[Boundary of the invariance statement]
At $\eta=0$, all channel energies vanish and the normalized allocation is
undefined.  Moreover, Theorem~\ref{thm:global_scale_invariance} is a
frozen-state, one-step statement.  Changing $\eta$ generally changes
$s_{t+1}$, so later gradients, moment estimates, and allocations need not
remain invariant.
\end{remark}

\subsection{Coordinate preconditioning as an exact allocation actuator}
\label{subsec:coordinate_allocation}

Consider the coordinatewise rule
\begin{equation}
    \Delta\theta_i(p_c)
    =-\eta\frac{m_i}{(v_i+\varepsilon)^{p_c}},
    \qquad
    v_i\geq0,\quad \varepsilon>0,\quad \eta\neq0.
    \label{eq:coordinate_rule}
\end{equation}
All quantities except $p_c$ are frozen.  Let
\begin{equation}
    \mathcal{I}:=\{i:m_i\neq0\},
    \qquad
    r_i:=v_i+\varepsilon>0,
    \qquad
    \ell_i:=\log r_i.
    \label{eq:active_coordinates}
\end{equation}
Assume $\mathcal{I}\neq\varnothing$.  The coordinate energy allocation is
\begin{equation}
    a_i(p_c)
    :=\frac{m_i^2r_i^{-2p_c}}
    {\sum_{j\in\mathcal{I}}m_j^2r_j^{-2p_c}},
    \qquad i\in\mathcal{I}.
    \label{eq:coordinate_share}
\end{equation}

\begin{theorem}[Affine log-odds control law]
\label{thm:coordinate_log_odds}
For any $i,j\in\mathcal{I}$,
\begin{equation}
    \log\frac{a_i(p_c)}{a_j(p_c)}
    =\log\frac{m_i^2}{m_j^2}
    +2p_c\log\frac{r_j}{r_i},
    \label{eq:coordinate_log_odds}
\end{equation}
and hence
\begin{equation}
    \frac{\partial}{\partial p_c}
    \log\frac{a_i(p_c)}{a_j(p_c)}
    =2\log\frac{r_j}{r_i}.
    \label{eq:coordinate_log_odds_derivative}
\end{equation}
Equivalently, the amplitude ratio obeys
\begin{equation}
    \log\frac{|\Delta\theta_i(p_c)|}
    {|\Delta\theta_j(p_c)|}
    =\log\frac{|m_i|}{|m_j|}
    +p_c\log\frac{r_j}{r_i}.
    \label{eq:coordinate_amplitude_ratio}
\end{equation}
Therefore $p_c$ locally actuates the pairwise allocation log-odds if and only
if $v_i\neq v_j$.
\end{theorem}

\begin{proof}
The normalizing denominator in \eqref{eq:coordinate_share} cancels in the
ratio, giving
\begin{equation}
    \frac{a_i(p_c)}{a_j(p_c)}
    =\frac{m_i^2}{m_j^2}
     \left(\frac{r_j}{r_i}\right)^{2p_c}.
\end{equation}
Taking logarithms proves \eqref{eq:coordinate_log_odds}; differentiation
proves \eqref{eq:coordinate_log_odds_derivative}.  Applying the same argument
directly to \eqref{eq:coordinate_rule} proves
\eqref{eq:coordinate_amplitude_ratio}.  Since $r_i=r_j$ exactly when
$v_i=v_j$, the final claim follows from
Definition~\ref{def:instantaneous_actuation}.
\end{proof}

\begin{corollary}[Exact pairwise reachability]
\label{cor:coordinate_reachability}
If $v_i\neq v_j$ and $p_c\in\mathbb{R}$, then for every target
$y^\star\in\mathbb{R}$ there exists a unique control
\begin{equation}
    p_c^\star
    =\frac{y^\star-\log(m_i^2/m_j^2)}
    {2\log(r_j/r_i)}
    \label{eq:coordinate_inverse}
\end{equation}
such that
$\log(a_i(p_c^\star)/a_j(p_c^\star))=y^\star$.  If $p_c$ is restricted to an
interval $[p_-,p_+]$, the reachable set is exactly the closed interval whose
endpoints are the values of \eqref{eq:coordinate_log_odds} at $p_-$ and
$p_+$.
\end{corollary}

\begin{proof}
Under $v_i\neq v_j$, the right-hand side of
\eqref{eq:coordinate_log_odds} is an affine function of $p_c$ with nonzero
slope.  It is therefore a bijection from $\mathbb{R}$ to $\mathbb{R}$, with
inverse \eqref{eq:coordinate_inverse}.  The image of a closed interval under a
nonconstant affine function is the interval determined by its endpoint
values.
\end{proof}

The pairwise law extends to the entire coordinate distribution.

\begin{proposition}[Replicator identity and monotone moment shift]
\label{prop:replicator_identity}
Let
\begin{equation}
    \bar\ell(p_c):=\sum_{i\in\mathcal{I}}a_i(p_c)\ell_i.
    \label{eq:mean_log_moment}
\end{equation}
Then
\begin{align}
    \frac{d a_i}{dp_c}
    &=-2a_i\bigl(\ell_i-\bar\ell\bigr),
    \label{eq:replicator}\\
    \frac{d\bar\ell}{dp_c}
    &=-2\operatorname{Var}_{a(p_c)}(\ell)
    \leq0.
    \label{eq:monotone_log_moment}
\end{align}
The inequality in \eqref{eq:monotone_log_moment} is strict exactly when the
$\ell_i$ are not all equal on $\mathcal{I}$.
\end{proposition}

\begin{proof}
Write $e_i(p_c)=m_i^2e^{-2p_c\ell_i}$ and
$Z(p_c)=\sum_{j\in\mathcal{I}}e_j(p_c)$, so that $a_i=e_i/Z$.  Since
$e_i'=-2\ell_i e_i$ and
\begin{equation}
    \frac{Z'}{Z}
    =-2\sum_{j\in\mathcal{I}}a_j\ell_j
    =-2\bar\ell,
\end{equation}
logarithmic differentiation yields
\eqref{eq:replicator}.  Multiplying \eqref{eq:replicator} by $\ell_i$ and
summing gives
\begin{equation}
    \bar\ell'
    =-2\left(\sum_i a_i\ell_i^2-\bar\ell^2\right)
    =-2\operatorname{Var}_{a(p_c)}(\ell).
\end{equation}
All $a_i$ are positive on $\mathcal{I}$, so the variance vanishes exactly when
all active $\ell_i$ are equal.
\end{proof}

Proposition~\ref{prop:replicator_identity} states only that increasing $p_c$
shifts update energy toward coordinates with smaller $v_i+\varepsilon$.  It
does not assert that those coordinates are more predictive, more causal, or
better for generalization.

\subsection{Affine spectral allocation and rank-one control}
\label{subsec:affine_spectral_allocation}

Let $M_W\in\mathbb{R}^{m\times d}$ be a weight-update statistic and
$m_b\in\mathbb{R}^{m}$ the corresponding bias statistic.  For a width
$d>1$, define
\begin{equation}
    \alpha(p_\alpha):=d^{p_\alpha}>0,
    \qquad
    A_{p_\alpha}:=\bigl[M_W,\,\alpha(p_\alpha)m_b\bigr],
    \label{eq:affine_matrix}
\end{equation}
and its left Gram matrix
\begin{equation}
    H_{p_\alpha}:=A_{p_\alpha}A_{p_\alpha}^{\mathsf T}
    =M_WM_W^{\mathsf T}+d^{2p_\alpha}m_bm_b^{\mathsf T}.
    \label{eq:affine_gram}
\end{equation}

\begin{theorem}[Rank-one positive-semidefinite actuation]
\label{thm:rank_one_psd}
The affine control satisfies
\begin{equation}
    \frac{\partial H_{p_\alpha}}{\partial p_\alpha}
    =2(\log d)d^{2p_\alpha}m_bm_b^{\mathsf T}
    \succeq0.
    \label{eq:gram_derivative}
\end{equation}
If $m_b\neq0$, the derivative has rank one.  More precisely, for every
$x\in\mathbb{R}^{m}$,
\begin{equation}
    \frac{\partial}{\partial p_\alpha}
    \bigl(x^{\mathsf T}H_{p_\alpha}x\bigr)
    =2(\log d)d^{2p_\alpha}(x^{\mathsf T}m_b)^2,
    \label{eq:quadratic_response}
\end{equation}
which is positive exactly when $x^{\mathsf T}m_b\neq0$.
\end{theorem}

\begin{proof}
Differentiate \eqref{eq:affine_gram} and use
$\frac{d}{dp_\alpha}d^{2p_\alpha}=2(\log d)d^{2p_\alpha}$ to obtain
\eqref{eq:gram_derivative}.  An outer product $m_bm_b^{\mathsf T}$ is
positive semidefinite, because
$x^{\mathsf T}m_bm_b^{\mathsf T}x=(x^{\mathsf T}m_b)^2\geq0$; it has rank one
when $m_b\neq0$.  The same quadratic-form identity proves
\eqref{eq:quadratic_response}.
\end{proof}

\begin{corollary}[Monotone eigenvalue response]
\label{cor:eigenvalue_response}
If $p_2\geq p_1$, then
$H_{p_2}\succeq H_{p_1}$ and every ordered eigenvalue satisfies
\begin{equation}
    \lambda_i(H_{p_2})\geq\lambda_i(H_{p_1}).
    \label{eq:eigenvalue_monotonicity}
\end{equation}
If $\lambda_i(H_{p_\alpha})$ is simple with unit eigenvector $u_i$, then
\begin{equation}
    \frac{d\lambda_i(H_{p_\alpha})}{dp_\alpha}
    =2(\log d)d^{2p_\alpha}(u_i^{\mathsf T}m_b)^2.
    \label{eq:eigenvalue_derivative}
\end{equation}
\end{corollary}

\begin{proof}
From \eqref{eq:affine_gram},
\begin{equation}
    H_{p_2}-H_{p_1}
    =\bigl(d^{2p_2}-d^{2p_1}\bigr)m_bm_b^{\mathsf T}\succeq0.
\end{equation}
The eigenvalue inequality follows from the Courant--Fischer min--max theorem.
For a simple eigenvalue, the Hellmann--Feynman identity gives
$\lambda_i'=u_i^{\mathsf T}H'u_i$; substituting
\eqref{eq:gram_derivative} proves \eqref{eq:eigenvalue_derivative}.
\end{proof}

The raw contribution of the affine column has an exact scalar control law.

\begin{proposition}[Logistic law for raw affine participation]
\label{prop:raw_affine_participation}
Assume $\lVert M_W\rVert_F>0$ and $\lVert m_b\rVert_2>0$.  Define
\begin{equation}
    \rho(p_\alpha)
    :=\frac{d^{2p_\alpha}\lVert m_b\rVert_2^2}
    {\lVert M_W\rVert_F^2+d^{2p_\alpha}\lVert m_b\rVert_2^2}.
    \label{eq:raw_bias_share}
\end{equation}
Then $\rho:\mathbb{R}\to(0,1)$ is strictly increasing and
\begin{align}
    \log\frac{\rho(p_\alpha)}{1-\rho(p_\alpha)}
    &=\log\frac{\lVert m_b\rVert_2^2}{\lVert M_W\rVert_F^2}
      +2p_\alpha\log d,
    \label{eq:raw_bias_logit}\\
    \frac{d\rho}{dp_\alpha}
    &=2(\log d)\rho(1-\rho)>0.
    \label{eq:raw_bias_derivative}
\end{align}
For any target $\rho^\star\in(0,1)$, the unique inverse is
\begin{equation}
    p_\alpha^\star
    =\frac{1}{2\log d}
      \left[
      \log\frac{\rho^\star}{1-\rho^\star}
      -\log\frac{\lVert m_b\rVert_2^2}{\lVert M_W\rVert_F^2}
      \right].
    \label{eq:raw_bias_inverse}
\end{equation}
\end{proposition}

\begin{proof}
Dividing \eqref{eq:raw_bias_share} by its complement gives
\begin{equation}
    \frac{\rho}{1-\rho}
    =d^{2p_\alpha}
      \frac{\lVert m_b\rVert_2^2}{\lVert M_W\rVert_F^2}.
\end{equation}
Taking logarithms proves \eqref{eq:raw_bias_logit}.  Differentiating the
logistic representation proves \eqref{eq:raw_bias_derivative}.  Because
$d>1$, the logit is an affine bijection from $\mathbb{R}$ to $\mathbb{R}$;
solving it for $p_\alpha$ proves \eqref{eq:raw_bias_inverse}.
\end{proof}

\begin{remark}[Raw participation is not a physical update share]
The quantity $\rho$ is the fraction of the squared Frobenius norm of the
\emph{input matrix} $A_{p_\alpha}$ carried by its last column.  A nonlinear
spectral map, global normalization, clipping rule, or output-side rescaling
can change the fraction carried by the final physical bias update.  Therefore
Proposition~\ref{prop:raw_affine_participation} must not be read as a theorem
about the final bias-update norm.
\end{remark}

\subsection{A two-input allocation map}
\label{subsec:unified_control_map}

The preceding control laws can be combined without conflating their targets.
Fix active coordinates $i,j$ with $v_i\neq v_j$, and assume the nondegeneracy
conditions of Proposition~\ref{prop:raw_affine_participation}.  Define
\begin{equation}
    y_c:=\log\frac{a_i}{a_j},
    \qquad
    y_\alpha:=\log\frac{\rho}{1-\rho}.
    \label{eq:two_outputs}
\end{equation}

\begin{theorem}[Exact frozen-state allocation coordinates]
\label{thm:two_input_control}
For independent controls $(p_c,p_\alpha)\in\mathbb{R}^2$,
\begin{equation}
    \begin{bmatrix}y_c\\y_\alpha\end{bmatrix}
    =
    \begin{bmatrix}
      \log(m_i^2/m_j^2)\\
      \log(\lVert m_b\rVert_2^2/\lVert M_W\rVert_F^2)
    \end{bmatrix}
    +
    \begin{bmatrix}
      2\log(r_j/r_i) & 0\\
      0 & 2\log d
    \end{bmatrix}
    \begin{bmatrix}p_c\\p_\alpha\end{bmatrix}.
    \label{eq:two_input_affine_map}
\end{equation}
The Jacobian in \eqref{eq:two_input_affine_map} is nonsingular.  Hence the
frozen-state map is a global bijection from $\mathbb{R}^2$ to
$\mathbb{R}^2$ in these two output coordinates.  On rectangular bounded
control domains, its reachable set is exactly the corresponding affine
rectangle in $(y_c,y_\alpha)$-space.
\end{theorem}

\begin{proof}
Equation \eqref{eq:two_input_affine_map} is obtained by stacking
\eqref{eq:coordinate_log_odds} and \eqref{eq:raw_bias_logit}.  Its determinant
is
\begin{equation}
    4\log(r_j/r_i)\log d\neq0
\end{equation}
because $v_i\neq v_j$ and $d>1$.  An affine map with an invertible linear part
is a global bijection.  The image of a Cartesian product of intervals under
this diagonal affine map is the Cartesian product of the two scalar image
intervals.
\end{proof}

Theorem~\ref{thm:two_input_control} is an exact result about two selected
instantaneous allocation coordinates.  It neither implies control of all
parameter groups nor establishes controllability of the trajectory
$s_0,s_1,\ldots$.

\subsection{Why the affine channel is structurally special}
\label{subsec:affine_rank_one_structure}

Consider an affine layer with mini-batch input
$X\in\mathbb{R}^{N\times d}$, weight $W\in\mathbb{R}^{m\times d}$, bias
$b\in\mathbb{R}^{m}$, and pre-activation
\begin{equation}
    Z=XW^{\mathsf T}+\mathbf{1}b^{\mathsf T}.
    \label{eq:affine_layer}
\end{equation}
For an empirical loss $L=N^{-1}\sum_{n=1}^{N}\ell_n$, let
$D\in\mathbb{R}^{N\times m}$ have rows
$D_{n,:}=\nabla_{Z_{n,:}}\ell_n^{\mathsf T}$.  Then
\begin{equation}
    G_W:=\nabla_WL=\frac{1}{N}D^{\mathsf T}X,
    \qquad
    g_b:=\nabla_bL=\frac{1}{N}D^{\mathsf T}\mathbf{1}.
    \label{eq:affine_gradients}
\end{equation}
Define the batch mean and centered input by
\begin{equation}
    \mu:=\frac{1}{N}X^{\mathsf T}\mathbf{1},
    \qquad
    X_c:=X-\mathbf{1}\mu^{\mathsf T},
    \qquad
    G_c:=\frac{1}{N}D^{\mathsf T}X_c.
    \label{eq:centered_input}
\end{equation}

\begin{proposition}[Exact rank-one affine decomposition]
\label{prop:rank_one_gradient}
The weight gradient decomposes as
\begin{equation}
    G_W=G_c+g_b\mu^{\mathsf T}.
    \label{eq:weight_gradient_decomposition}
\end{equation}
Consequently, for every $\alpha>0$,
\begin{equation}
    [G_W,\,\alpha g_b]
    =[G_c,\,0]+g_b[\mu^{\mathsf T},\,\alpha].
    \label{eq:augmented_gradient_decomposition}
\end{equation}
The second term in \eqref{eq:augmented_gradient_decomposition} has rank at
most one, and has rank exactly one if $g_b\neq0$.
\end{proposition}

\begin{proof}
Substitute $X=X_c+\mathbf{1}\mu^{\mathsf T}$ into
\eqref{eq:affine_gradients}:
\begin{align}
    G_W
    &=\frac{1}{N}D^{\mathsf T}X_c
      +\frac{1}{N}D^{\mathsf T}\mathbf{1}\mu^{\mathsf T}\\
    &=G_c+g_b\mu^{\mathsf T}.
\end{align}
Appending the scaled bias column yields
\eqref{eq:augmented_gradient_decomposition}.  Its second term is an outer
product of $g_b$ and the nonzero row vector $[\mu^{\mathsf T},\alpha]$, so its
rank is zero when $g_b=0$ and one otherwise.
\end{proof}

The same factorization does not automatically survive temporal averaging when
mini-batch means vary.  The exact replacement is as follows.

\begin{proposition}[Momentum decomposition with a mean-drift remainder]
\label{prop:momentum_remainder}
Let $s$ index a finite history, let $c_s\in\mathbb{R}$ be scalar aggregation
coefficients, and write the per-step decomposition as
\begin{equation}
    G_{W,s}=G_{c,s}+g_{b,s}\mu_s^{\mathsf T}.
\end{equation}
Define
\begin{equation}
    M_W:=\sum_s c_sG_{W,s},
    \qquad
    M_c:=\sum_s c_sG_{c,s},
    \qquad
    m_b:=\sum_s c_sg_{b,s}.
    \label{eq:aggregated_momenta}
\end{equation}
For any reference mean $\bar\mu\in\mathbb{R}^{d}$,
\begin{equation}
    M_W=M_c+m_b\bar\mu^{\mathsf T}+R_{\bar\mu},
    \qquad
    R_{\bar\mu}
    :=\sum_s c_sg_{b,s}(\mu_s-\bar\mu)^{\mathsf T},
    \label{eq:momentum_remainder}
\end{equation}
with the bound
\begin{equation}
    \lVert R_{\bar\mu}\rVert_F
    \leq\sum_s |c_s|\,\lVert g_{b,s}\rVert_2
                   \lVert\mu_s-\bar\mu\rVert_2.
    \label{eq:momentum_remainder_bound}
\end{equation}
In particular, if all $\mu_s$ equal a common $\mu$, choosing
$\bar\mu=\mu$ makes the rank-one factorization exact.
\end{proposition}

\begin{proof}
Summing the per-step identities gives
\begin{equation}
    M_W=M_c+\sum_s c_sg_{b,s}\mu_s^{\mathsf T}.
\end{equation}
Add and subtract
$\sum_s c_sg_{b,s}\bar\mu^{\mathsf T}=m_b\bar\mu^{\mathsf T}$ to obtain
\eqref{eq:momentum_remainder}.  The triangle inequality and
$\lVert xy^{\mathsf T}\rVert_F=\lVert x\rVert_2\lVert y\rVert_2$ yield
\eqref{eq:momentum_remainder_bound}.  The common-mean case makes every term in
$R_{\bar\mu}$ vanish.
\end{proof}

Thus the affine channel is exactly rank one at a fixed batch.  For momentum or
other temporal aggregates, it remains an exact rank-one term plus an explicit
mean-drift remainder; calling it exactly rank one without controlling that
remainder would require an additional assumption.

\subsection{Conditional selective attenuation under ideal spectral whitening}
\label{subsec:selective_attenuation}

Let
\begin{equation}
    C:=M_WM_W^{\mathsf T},
    \qquad
    H_\alpha:=C+\alpha^2m_bm_b^{\mathsf T}.
    \label{eq:alpha_gram}
\end{equation}
In this subsection assume $C\succ0$, so $H_\alpha\succ0$ for every
$\alpha\geq0$, and define the ideal left-whitening operator
\begin{equation}
    P_\alpha:=H_\alpha^{-1/2}.
    \label{eq:ideal_preconditioner}
\end{equation}

\begin{theorem}[Loewner-monotone gain attenuation]
\label{thm:loewner_attenuation}
If $0\leq\alpha_1\leq\alpha_2$, then
\begin{equation}
    H_{\alpha_1}\preceq H_{\alpha_2},
    \qquad
    P_{\alpha_2}\preceq P_{\alpha_1}.
    \label{eq:loewner_attenuation}
\end{equation}
Consequently, for every $x\in\mathbb{R}^{m}$,
\begin{equation}
    x^{\mathsf T}P_{\alpha_2}x
    \leq x^{\mathsf T}P_{\alpha_1}x.
    \label{eq:quadratic_gain_attenuation}
\end{equation}
\end{theorem}

\begin{proof}
The first order relation follows from
\begin{equation}
    H_{\alpha_2}-H_{\alpha_1}
    =(\alpha_2^2-\alpha_1^2)m_bm_b^{\mathsf T}\succeq0.
\end{equation}
Matrix inversion reverses the Loewner order on positive-definite matrices, so
$H_{\alpha_2}^{-1}\preceq H_{\alpha_1}^{-1}$.  The principal square-root map
is operator monotone on the positive-semidefinite cone; applying it to the
last inequality gives
$H_{\alpha_2}^{-1/2}\preceq H_{\alpha_1}^{-1/2}$.  Taking quadratic forms
proves \eqref{eq:quadratic_gain_attenuation}.
\end{proof}

Theorem~\ref{thm:loewner_attenuation} is global in the Loewner sense.  A
direction-by-direction selective formula requires an alignment assumption.

\begin{theorem}[Exact selective formula under eigenvector alignment]
\label{thm:aligned_selective_attenuation}
Assume $m_b\neq0$, let $u:=m_b/\lVert m_b\rVert_2$, and suppose
\begin{equation}
    Cu=\lambda_bu,
    \qquad \lambda_b>0.
    \label{eq:alignment_assumption}
\end{equation}
Then
\begin{equation}
    P_\alpha u
    =\frac{1}{\sqrt{\lambda_b+\alpha^2\lVert m_b\rVert_2^2}}u.
    \label{eq:bias_direction_gain}
\end{equation}
The scalar gain in \eqref{eq:bias_direction_gain} is strictly decreasing in
$\alpha>0$.  If $q\perp u$ is a unit eigenvector of $C$ with eigenvalue
$\lambda_q$, then
\begin{equation}
    P_\alpha q=\lambda_q^{-1/2}q,
    \label{eq:orthogonal_gain}
\end{equation}
which is independent of $\alpha$.
\end{theorem}

\begin{proof}
Since $m_b=\lVert m_b\rVert_2u$,
\begin{equation}
    H_\alpha u
    =\bigl(\lambda_b+\alpha^2\lVert m_b\rVert_2^2\bigr)u.
\end{equation}
Functional calculus for the positive-definite matrix $H_\alpha$ gives
\eqref{eq:bias_direction_gain}.  Its scalar derivative equals
\begin{equation}
    -\frac{\alpha\lVert m_b\rVert_2^2}
    {(\lambda_b+\alpha^2\lVert m_b\rVert_2^2)^{3/2}}<0.
\end{equation}
For $q\perp u$, one has $m_bm_b^{\mathsf T}q=0$, hence
$H_\alpha q=Cq=\lambda_qq$, which proves
\eqref{eq:orthogonal_gain}.
\end{proof}

The attenuation in Theorem~\ref{thm:aligned_selective_attenuation} concerns
the \emph{preconditioner gain}.  The following calculation shows why this
must be separated from the physical update carried by the scaled affine
column.

\begin{proposition}[Scaled bias-column response]
\label{prop:scaled_bias_response}
Under the assumptions of
Theorem~\ref{thm:aligned_selective_attenuation}, the last column of
$P_\alpha A_\alpha$ has norm
\begin{equation}
    \lVert P_\alpha(\alpha m_b)\rVert_2
    =\frac{\alpha\lVert m_b\rVert_2}
    {\sqrt{\lambda_b+\alpha^2\lVert m_b\rVert_2^2}}.
    \label{eq:physical_bias_response}
\end{equation}
This quantity is strictly increasing for $\alpha>0$ and converges to $1$ as
$\alpha\to\infty$.
\end{proposition}

\begin{proof}
Equation \eqref{eq:physical_bias_response} follows by multiplying
\eqref{eq:bias_direction_gain} by $\alpha\lVert m_b\rVert_2$.  With
$\beta:=\lVert m_b\rVert_2$, differentiation gives
\begin{equation}
    \frac{d}{d\alpha}
    \frac{\alpha\beta}{\sqrt{\lambda_b+\alpha^2\beta^2}}
    =\frac{\beta\lambda_b}
    {(\lambda_b+\alpha^2\beta^2)^{3/2}}>0.
\end{equation}
The limit follows after dividing numerator and denominator by $\alpha\beta$.
\end{proof}

Therefore increasing $\alpha$ rigorously attenuates the ideal inverse-square-
root gain along the aligned affine direction, but it does not by itself imply
that an algorithm emitting the scaled last column without division by
$\alpha$ produces a smaller physical bias update.  Such a claim requires the
complete output map, including clipping, normalization, and decoding.

\subsection{Finite-step spectral maps preserve modal subspaces}
\label{subsec:finite_spectral_maps}

Let $A\in\mathbb{R}^{m\times n}$ have compact singular value decomposition
\begin{equation}
    A=U\operatorname{diag}(\sigma_1,\ldots,\sigma_r)V^{\mathsf T},
    \qquad \sigma_i>0.
    \label{eq:compact_svd}
\end{equation}
Any singular-value map of the form
\begin{equation}
    \mathcal{S}_\phi(A)
    :=U\operatorname{diag}\bigl(\phi(\sigma_1),\ldots,
        \phi(\sigma_r)\bigr)V^{\mathsf T}
    \label{eq:spectral_map}
\end{equation}
preserves the left and right singular subspaces of $A$.  The polar map uses
$\phi(\sigma)=1$.  Exact-SVD capped inverse maps use a different scalar
function $\phi$, generally depending also on a normalization by the largest
singular value.  These are different spectral response laws applied to the
same affine object; equality between them is not assumed.

For a finite Newton--Schulz-type polynomial realization, choose a scale
$\kappa>0$ and define
\begin{equation}
    X_0:=\frac{A}{\kappa},
    \qquad
    X_{k+1}:=aX_k+bX_kX_k^{\mathsf T}X_k
      +cX_kX_k^{\mathsf T}X_kX_k^{\mathsf T}X_k,
    \label{eq:ns_iteration}
\end{equation}
where $a,b,c\in\mathbb{R}$ are fixed coefficients.

\begin{proposition}[Exact modal invariance of finite polynomial iterations]
\label{prop:ns_modal_invariance}
For every integer $k\geq0$,
\begin{equation}
    X_k
    =U\operatorname{diag}
      \left(f_k\left(\frac{\sigma_1}{\kappa}\right),\ldots,
            f_k\left(\frac{\sigma_r}{\kappa}\right)\right)V^{\mathsf T},
    \label{eq:ns_scalarization}
\end{equation}
where
\begin{equation}
    f_0(x)=x,
    \qquad
    f_{k+1}(x)=af_k(x)+bf_k(x)^3+cf_k(x)^5.
    \label{eq:ns_scalar_recursion}
\end{equation}
Thus every finite iterate preserves the singular-vector subspaces of $A$ and
changes only its singular values.
\end{proposition}

\begin{proof}
The formula holds for $k=0$ by \eqref{eq:compact_svd}.  Assume it holds for
$k$ and abbreviate the diagonal matrix in
\eqref{eq:ns_scalarization} by $F_k$.  Orthogonality of the columns of $U$ and
$V$ gives
\begin{align}
    X_kX_k^{\mathsf T}X_k
    &=UF_k^3V^{\mathsf T},\\
    X_kX_k^{\mathsf T}X_kX_k^{\mathsf T}X_k
    &=UF_k^5V^{\mathsf T}.
\end{align}
Substitution into \eqref{eq:ns_iteration} yields
$X_{k+1}=U(aF_k+bF_k^3+cF_k^5)V^{\mathsf T}$, which is precisely
\eqref{eq:ns_scalar_recursion}.  Induction completes the proof.
\end{proof}

\begin{corollary}[Finite-step error relative to the polar factor]
\label{cor:ns_polar_error}
Let $Q:=UV^{\mathsf T}$ be the compact polar factor of $A$.  Then
\begin{equation}
    \lVert X_K-Q\rVert_2
    =\max_{1\leq i\leq r}
      \left|f_K\left(\frac{\sigma_i}{\kappa}\right)-1\right|.
    \label{eq:ns_polar_error}
\end{equation}
Therefore convergence to the polar map follows only after proving that the
scalar recursion approaches $1$ uniformly over the normalized singular-value
range encountered by the algorithm.
\end{corollary}

\begin{proof}
Subtract $Q=UI_rV^{\mathsf T}$ from
\eqref{eq:ns_scalarization}.  The spectral norm is invariant under
multiplication by semi-orthogonal matrices, and the norm of a diagonal matrix
is the largest absolute diagonal entry, which proves
\eqref{eq:ns_polar_error}.
\end{proof}

Proposition~\ref{prop:ns_modal_invariance} is exact for every finite number of
steps and does not require convergence.  Corollary~\ref{cor:ns_polar_error}
also shows that modal invariance alone is insufficient to identify a finite
polynomial iterate with either the polar map or a capped regularized-inverse
map.

\subsection{Logical scope: allocation control is not a generalization theorem}
\label{subsec:logical_scope}

The preceding results establish exact input--allocation relations at a frozen
state and conditional statements about spectral gain.  They do not impose a
universal ordering on downstream performance.

\begin{proposition}[No task-independent performance ordering]
\label{prop:no_universal_performance}
Fix a parameter state $\theta$ and two distinct candidate updates
$\delta_1\neq\delta_2$.  Knowledge of their allocation vectors alone cannot
determine which update has lower evaluation risk over all smooth nonnegative
risks.
\end{proposition}

\begin{proof}
Consider the two smooth nonnegative functions
\begin{equation}
    R_1(\vartheta):=\lVert\vartheta-(\theta+\delta_1)\rVert_2^2,
    \qquad
    R_2(\vartheta):=\lVert\vartheta-(\theta+\delta_2)\rVert_2^2.
    \label{eq:reversing_risks}
\end{equation}
For $R_1$,
\begin{equation}
    R_1(\theta+\delta_1)=0
    <\lVert\delta_2-\delta_1\rVert_2^2
    =R_1(\theta+\delta_2).
\end{equation}
For $R_2$, the strict inequality is reversed.  Hence the same pair of
updates, and therefore the same pair of allocation vectors, can receive
opposite performance orderings under admissible risks.  Additional assumptions
linking the data distribution, model, risk, and update direction are necessary
for a generalization claim.
\end{proof}

The rigorous conclusion of this section is consequently limited to the
following statements.  A nonzero global scale controls instantaneous update
magnitude but cancels from normalized allocation.  The coordinate exponent
$p_c$ exactly controls pairwise coordinate log-odds whenever the corresponding
second moments differ.  The affine exponent $p_\alpha$ produces a rank-one
positive-semidefinite perturbation and exactly controls the raw affine-column
logit under nondegeneracy.  The affine gradient contains an exact rank-one
mean component at a fixed batch, with a quantified remainder under changing
batch means.  Ideal inverse-square-root gain decreases as the affine column is
strengthened, while the physical scaled-column response requires separate
analysis.  Finally, finite polynomial spectral iterations preserve modal
subspaces exactly, but their approximation quality is governed by their scalar
response on the realized singular-value range.

\section{Experiments: From Instantaneous Actuation to Trajectory-Level Behavior}
\label{sec:experiments}

Section~\ref{sec:allocation_theory} established frozen-state input--allocation laws.  The experiments
are therefore organized as an evidence ladder rather than as a flat benchmark
comparison.  We first test whether the proposed controls change allocation
while the state is held fixed.  We then ask how those interventions alter
training trajectories on a controlled benchmark, where stable, spurious, and
noise channels are known.  Finally, we report a descriptive transfer study on
four standard datasets.  This ordering is important: performance differences
alone cannot identify an allocation mechanism, while a one-step mechanism
check alone cannot establish useful trajectory-level behavior.

\subsection{Questions, claims, and evidence hierarchy}
\label{sec:exp_questions}

The experimental section is built around four questions.

\begin{itemize}
    \item \textbf{RQ1: Direct actuation.}  With the optimizer state and
    mini-batch fixed, does $\pc$ redistribute coordinate-channel command
    energy, and does $\pa$ obey the affine-column logit law?

    \item \textbf{RQ2: Trajectory-level utility.}  When the controls are used
    throughout training, which ranges improve final accuracy, loss, and
    worst-group accuracy across random seeds?

    \item \textbf{RQ3: Failure boundary.}  Is performance monotone in control
    strength, or does excessive allocation pressure create a collapse regime?

    \item \textbf{RQ4: External validity.}  Does the useful range identified
    by the controlled benchmark remain competitive across classification and
    language-modeling tasks?
\end{itemize}

\begin{table}[t]
  \centering
  \caption{Evidence hierarchy and the maximum claim supported by each block.}
  \label{tab:evidence-hierarchy}
  \small
  \begin{tabularx}{\textwidth}{@{}p{2.15cm}p{3.05cm}p{3.4cm}L@{}}
    \toprule
    Block & Unit of comparison & Primary observable & Admissible conclusion \\
    \midrule
    Same-state replay & Fixed state, batch, and moments & Command shares, affine raw-share logit, gain/output ratios & Direct instantaneous actuation; no trajectory or generalization claim \\
    Controlled training & Five paired seeds, 42--46 & Final test accuracy/loss, tail loss, worst-group accuracy & Reproducible trajectory-level association on the controlled benchmark \\
    Diagnostics & Epoch checkpoints and known channel labels & Stable, spurious, noise, and bias allocation & Mechanistic consistency and failure-mode interpretation \\
    Transfer suite & Four datasets, seed 42 & Validation-selected accuracy or perplexity; within-task rank & Descriptive external validity only; no significance claim \\
    \bottomrule
  \end{tabularx}
\end{table}

\subsection{Experimental design}
\label{sec:exp_design}

\subsubsection{Controlled allocation benchmark}
\label{subsec:controlled_benchmark}

The confirmatory benchmark is a binary classification problem with
$8{,}192$ training and $8{,}192$ test examples.  Its $3{,}124$ input
coordinates are partitioned into four known blocks: 4 dense coordinates, 60
stable sparse coordinates, 60 spurious sparse coordinates, and 3,000 noise
coordinates.  The stable signal has correlation $1.0$ in both splits.  The
spurious block has correlation $0.9$ in training and $0$ at test time, creating
an explicit distribution shift.  A one-hidden-layer ReLU network with 64
hidden units is trained for 50 epochs with batch size 128.  The class prior is
$0.6$.

We compare three intervention families.  Coordinate control uses
$\pc\in\{0,0.1,\ldots,0.5\}$.  The spectral reference is Muon with momentum
$0.95$, Nesterov updates, five Newton--Schulz steps, and bfloat16 matrix
arithmetic.  Affine experiments include a probe-only condition at $\pa=1$ and
joint affine updates with
$\pa\in\{0,0.25,\ldots,2\}$.  Adam-style runs use
$\beta_1=0.9$, $\beta_2=0.999$, $\varepsilon=10^{-8}$, and zero weight decay.
Both the Adam-style and spectral learning rates are $10^{-3}$.  The joint bias
cap is $4.0$.

\subsubsection{Same-state replay}
\label{subsec:same_state_replay}

Trajectory comparisons mix two effects: the direct response of the optimizer
at the current state and the indirect effect of arriving at a different later
state.  We separate them by replaying all control values at checkpoints from
two source trajectories, PAdam $\pc=0.3$ and Joint $\pa=1$.  For each of five
seeds and each of 50 checkpoints, the state, mini-batch, gradients, and moments
are fixed while the candidate control is changed.  This yields
$5\times2\times50=500$ frozen states for each control value.  Checkpoints are
treated as repeated mechanism probes, not as 500 independent statistical
replicates.

\subsubsection{Transfer suite and model-selection rule}
\label{subsec:transfer_suite}

The transfer suite contains IMDB and Rotten Tomatoes classification, together
with Penn Treebank (PTB) and WikiText-2 language modeling.  We compare AdamW,
Muon, affine probe $\pa=1$, and joint affine controls
$\pa\in\{0,0.5,1,1.5,2\}$.  Classification reports test accuracy at the
validation-selected checkpoint; language modeling reports test perplexity at
the validation-selected checkpoint.  Only seed 42 is available for this suite,
so its role is exploratory corroboration rather than confirmatory replication.

\begin{table}[t]
  \centering
  \caption{Compact experimental protocol.  Values are taken from the supplied run manifests and summaries.}
  \label{tab:protocol}
  \small
  \begin{tabularx}{\textwidth}{@{}p{3.0cm}LL@{}}
    \toprule
    Item & Controlled benchmark & Transfer suite \\
    \midrule
    Tasks & Synthetic binary classification with known stable, spurious, and noise channels & IMDB, Rotten Tomatoes, PTB, and WikiText-2 \\
    Repetition & Seeds 42--46 ($n=5$) & Seed 42 only \\
    Selection & Final epoch; oracle extrema reported only as diagnostics & Validation-selected checkpoint \\
    Control grid & $\pc=0{:}0.1{:}0.5$; $\pa=0{:}0.25{:}2$ & $\pa\in\{0,0.5,1,1.5,2\}$ plus AdamW, Muon, and probe \\
    Primary metric & Final test accuracy; test loss and worst-group accuracy are co-primary diagnostics & Accuracy for classification; perplexity for language modeling \\
    Statistical unit & Seed; report mean $\pm$ sample SD and paired differences & Dataset-level rank; descriptive only \\
    \bottomrule
  \end{tabularx}
  \par\vspace{2pt}\footnotesize\raggedright
  The supplied controlled archive contains five completed seeds even though the original configuration also lists future seeds 47--51.  The analysis does not treat missing seeds as completed runs.
\end{table}

\subsubsection{Reporting policy}
\label{subsec:reporting_policy}

For the controlled benchmark, each method is summarized by the arithmetic
mean and sample standard deviation over seeds.  Differences from PAdam
$\pc=0$ are paired by seed, and the win count records the number of seeds with
higher final test accuracy.  With only five pairs, five wins correspond to a
two-sided exact sign-test value of $2/2^5=0.0625$; we therefore report effect
sizes and consistency rather than attaching significance stars.  No
hyperparameter is selected using test performance for the transfer suite; the
reported checkpoint follows the stored validation-selection rule.

\subsection{Control sweeps reveal useful and failure regimes}
\label{sec:control_sweeps}

\begin{figure}[t]
  \centering
  \includegraphics[width=\textwidth]{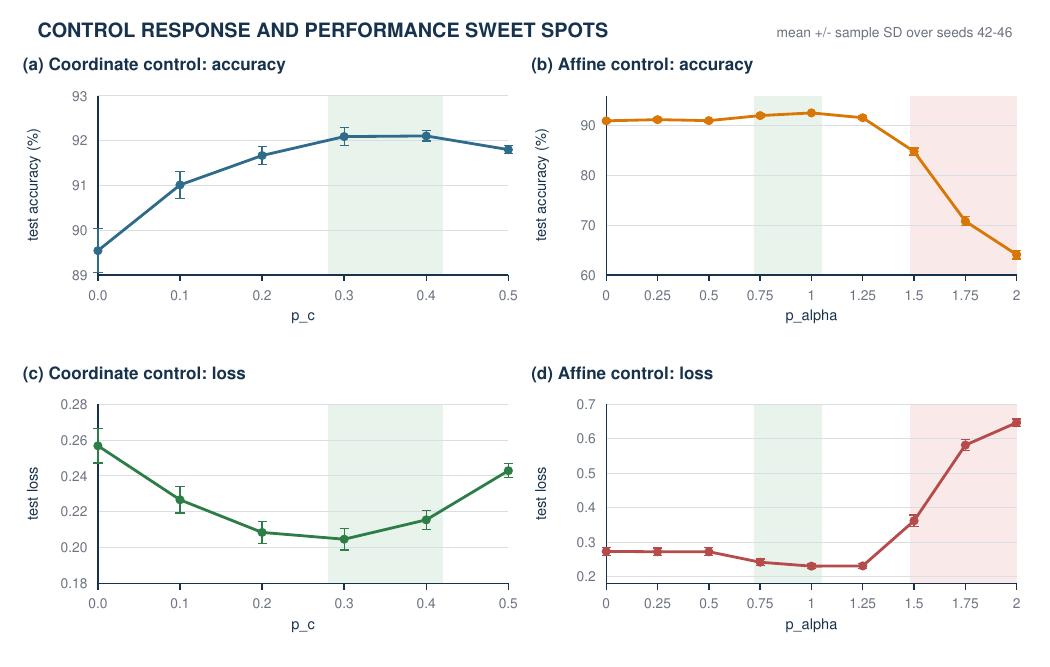}
  \caption{Five-seed control sweeps on the controlled benchmark.  Error bars
  are sample standard deviations across seeds 42--46.  Green shading marks the
  empirically useful range; red shading marks the high-$\pa$ collapse regime.
  These regions summarize the observed grid and are not confidence intervals.}
  \label{fig:control-sweeps}
\end{figure}

Coordinate control exhibits a broad intermediate optimum
(Fig.~\ref{fig:control-sweeps}a,c).  Final test accuracy rises from
$89.54\%$ at $\pc=0$ to $92.09\%$ at $\pc=0.3$ and $92.10\%$ at
$\pc=0.4$.  Mean test loss is minimized at $\pc=0.3$ ($0.2046$) and then
increases to $0.2429$ at $\pc=0.5$.  Thus $\pc$ is a genuine trajectory-level
intervention, but its performance effect is not monotone.  This is consistent
with Section~\ref{sec:allocation_theory}: the theorem identifies how allocation moves, not which
allocation must be optimal for a task.

Affine control has a narrower response window
(Fig.~\ref{fig:control-sweeps}b,d).  Joint $\pa=1$ reaches
$92.57\pm0.09\%$ accuracy and $0.2299$ loss.  Increasing the exponent to
$1.25$ already reduces accuracy, while $\pa=1.5$, $1.75$, and $2$ create
progressively severe underfitting.  At $\pa=2$, final training accuracy is only
about $71\%$ in the representative seed-42 run, so the $64.09\%$ mean test
accuracy is not an ordinary generalization trade-off; it is a training
collapse.  The full response curve should therefore remain in the main text,
not be hidden in an appendix.

\begin{table}[t]
  \centering
  \caption{Selected points from the five-seed controlled benchmark.  Higher is
  better for accuracy metrics and lower is better for loss.}
  \label{tab:controlled-main}
  \scriptsize
  \begin{tabular}{@{}lccccc@{}}
    \toprule
    Method & Control & Test acc. (\%) & Test loss & Worst (\%) & $\Delta$ (pp) \\
    \midrule
    PAdam & $\pc=0$ & $89.54\pm0.49$ & 0.2568 & 80.15 & $+0.00$ \\
    \rowcolor{bestgreen}
    PAdam & $\pc=0.3$ & $92.09\pm0.20$ & \best{0.2046} & 84.18 & $+2.55$ \\
    PAdam & $\pc=0.4$ & $92.10\pm0.12$ & 0.2154 & 84.27 & $+2.56$ \\
    Muon & -- & $91.33\pm0.36$ & 0.2661 & 80.82 & $+1.79$ \\
    \rowcolor{softblue}
    Affine probe & $\pa=1$ & \best{$92.66\pm0.20$} & 0.2353 & \best{85.98} & $+3.12$ \\
    Joint & $\pa=0$ & $90.98\pm0.26$ & 0.2728 & 80.02 & $+1.43$ \\
    Joint & $\pa=0.75$ & $92.03\pm0.31$ & 0.2413 & 83.09 & $+2.48$ \\
    \rowcolor{bestgreen}
    Joint & $\pa=1$ & \second{$92.57\pm0.09$} & 0.2299 & \second{85.93} & $+3.03$ \\
    Joint & $\pa=1.25$ & $91.60\pm0.40$ & 0.2306 & 84.47 & $+2.05$ \\
    \rowcolor{warningred}
    Joint & $\pa=1.5$ & $84.81\pm0.72$ & 0.3615 & 73.87 & $-4.74$ \\
    \rowcolor{warningred}
    Joint & $\pa=2$ & $64.09\pm0.90$ & 0.6463 & 57.65 & $-25.45$ \\
    \bottomrule
  \end{tabular}
  \par\vspace{2pt}\footnotesize\raggedright
  $\Delta$ accuracy is the paired mean difference from PAdam $\pc=0$.  All non-collapsed controlled settings shown above beat that baseline on accuracy in all five seeds; the collapsed settings do not.
\end{table}

\FloatBarrier

\subsection{Same-state replay verifies the proposed mechanisms}
\label{sec:mechanism_results}

\begin{figure}[t]
  \centering
  \includegraphics[width=\textwidth]{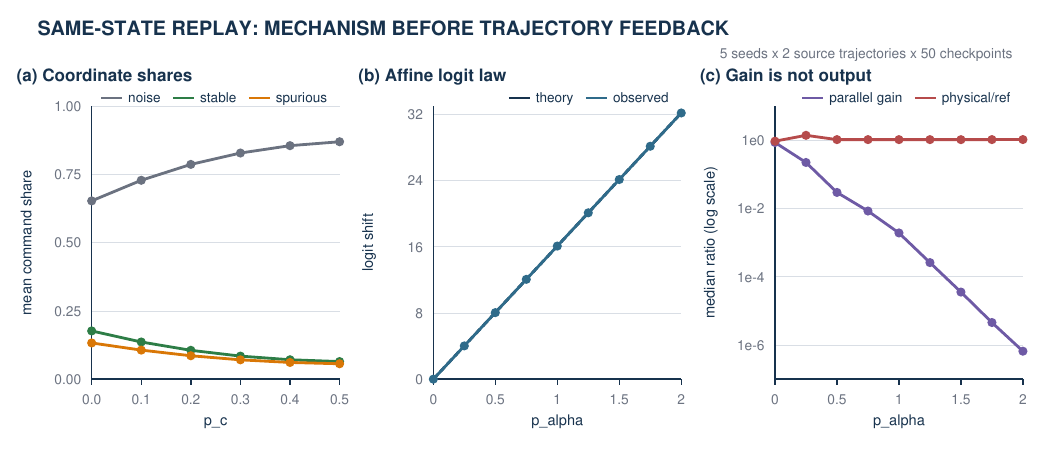}
  \caption{Same-state replay over 500 frozen states per control value.
  Panel (a) reports mean command-energy shares; dispersion across checkpoints
  is descriptive and omitted for readability.  Panel (b) subtracts each
  state's $\pa=0$ logit, so Section~\ref{sec:allocation_theory} predicts the common line
  $2\pa\log(3124)$.  Panel (c) separates attenuation parallel to the affine
  direction from the final physical-bias/reference ratio.}
  \label{fig:same-state-replay}
\end{figure}

\subsubsection{Coordinate redistribution}
\label{subsec:coordinate_replay}

Changing $\pc$ while the state is fixed produces a large and systematic
redistribution (Fig.~\ref{fig:same-state-replay}a).  The mean noise command share
increases from $0.653$ at $\pc=0$ to $0.870$ at $\pc=0.5$, while the stable
sparse share decreases from $0.177$ to $0.064$ and the spurious sparse share
decreases from $0.133$ to $0.056$.  These group aggregates are not expected to
be affine in $\pc$, because each group sums coordinates with heterogeneous
second moments.  Their smooth response nevertheless confirms that the control
acts before trajectory feedback.  It also explains why stronger control need
not improve performance indefinitely: low-moment noise coordinates can receive
increasing command energy together with useful sparse coordinates.

\subsubsection{Affine logit law and numerical verification}
\label{subsec:affine_replay}

For each frozen state, let $\rho(\pa)$ be the raw squared-energy share of the
scaled affine column.  Section~\ref{sec:allocation_theory} gives
\begin{equation}
  \operatorname{logit}\rho(\pa)-\operatorname{logit}\rho(0)
  =2\pa\log d.
  \label{eq:empirical-logit-law}
\end{equation}
Here $d=3124$, so the predicted slope is
$2\log(3124)=16.093739$.  A least-squares fit over all replay rows gives slope
$16.093742$, intercept $-1.65\times10^{-6}$, and
$R^2>0.9999999999$.  The maximum absolute deviation from the theoretical line
is $1.10\times10^{-3}$ and occurs only when $\rho$ is numerically saturated
near one.  This agreement is a verification of the implemented control law,
not evidence that large $\pa$ improves generalization.

\subsubsection{Preconditioner gain versus physical output}
\label{subsec:gain_output}

The median attenuation parallel to the affine direction falls by roughly six
orders of magnitude across the replay grid, while the physical bias update
relative to its Adam reference remains of order one
(Fig.~\ref{fig:same-state-replay}c).  This is the empirical counterpart of the
distinction proved in Section~\ref{sec:allocation_theory}: decreasing inverse-square-root gain does not
by itself imply a decreasing decoded bias update when the input column has
already been multiplied by $\alpha=d^{\pa}$ and the implementation applies
caps or output normalization.

\begin{table}[t]
  \centering
  \caption{Mechanism checks and their interpretation.}
  \label{tab:mechanism-checks}
  \small
  \begin{tabularx}{\textwidth}{@{}p{3.0cm}p{3.25cm}p{3.25cm}L@{}}
    \toprule
    Check & Theory or design target & Observation & Supported conclusion \\
    \midrule
    Coordinate replay & Nonzero response to $\pc$ at fixed state & Noise share $0.653\to0.870$; stable share $0.177\to0.064$ & $\pc$ directly changes group allocation; group utility is not implied \\
    Affine raw share & Logit slope $2\log d=16.093739$ & Fitted slope $16.093742$, $R^2>0.9999999999$ & Exact affine scaling law is implemented correctly \\
    Gain/output separation & Gain attenuation and decoded output are distinct & Parallel attenuation collapses; physical/reference remains $O(1)$ & Do not describe gain attenuation as automatic physical-bias suppression \\
    Trajectory response & Useful allocation need not be monotone in control & Intermediate optimum followed by high-$\pa$ collapse & Control strength requires a bounded operating region \\
    \bottomrule
  \end{tabularx}
\end{table}

\FloatBarrier

\subsection{Metric-specific optima}
\label{sec:robustness_results}

No setting wins every metric.  The affine probe leads mean accuracy and
worst-group accuracy; PAdam $\pc=0.3$ minimizes mean and P95 loss; and PAdam
$\pc=0.2$ minimizes non-collapsed P99 loss.  Reporting only accuracy would
therefore hide different optima for calibration- and tail-sensitive uses.

\begin{table}[t]
  \centering
  \caption{Metric-specific winners on the controlled benchmark.}
  \label{tab:metric-winners}
  \small
  \begin{tabular}{@{}lll r@{}}
    \toprule
    Metric & Direction & Best non-collapsed configuration & Value \\
    \midrule
    Final test accuracy & \upar & Affine probe $\pa=1$ & $92.66\%$ \\
    Final test loss & \downar & PAdam $\pc=0.3$ & $0.2046$ \\
    P95 sample loss & \downar & PAdam $\pc=0.3$ & $1.2316$ \\
    P99 sample loss & \downar & PAdam $\pc=0.2$ & $3.4135$ \\
    Worst-group accuracy & \upar & Affine probe $\pa=1$ & $85.98\%$ \\
    Elapsed time & \downar & Muon & $43.24$ s \\
    \bottomrule
  \end{tabular}
  \par\vspace{3pt}
  \footnotesize Joint $\pa=2$ has a numerically smaller P99 loss only because
  severe underfitting compresses predictions; it is excluded from the
  non-collapsed comparison.
\end{table}

The best methods also have small oracle-to-final accuracy gaps: approximately
$0.01$ percentage points for the affine probe and $0.03$ percentage points for
Joint $\pa=1$.  Their ranking is therefore not driven by a single early test
checkpoint.  Oracle values remain diagnostics only and are not used as the
selection rule.

\subsection{Cross-dataset transfer is promising but descriptive}
\label{sec:cross_dataset}

\begin{figure}[t]
  \centering
  \includegraphics[width=\textwidth]{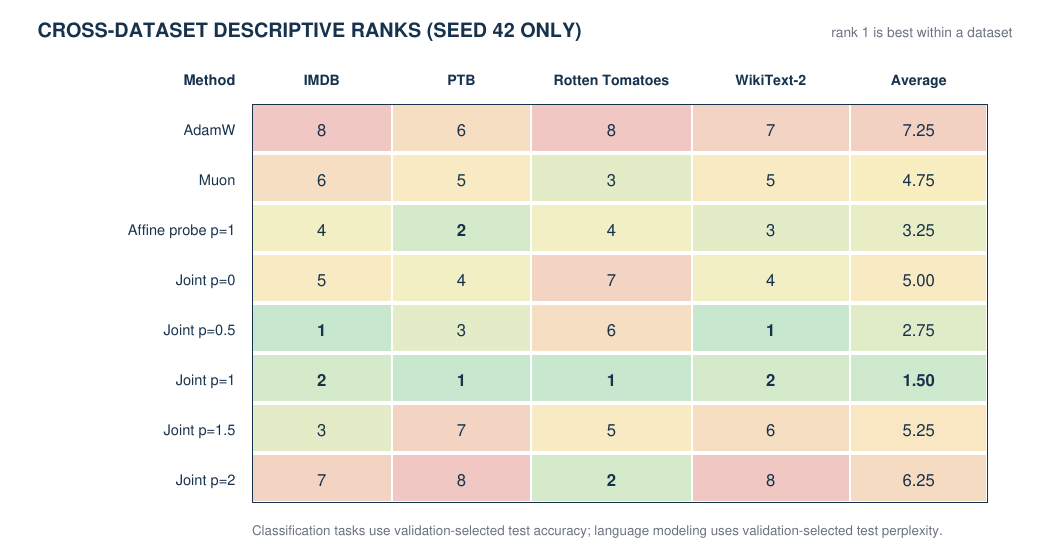}
  \caption{Within-dataset ranks for the seed-42 transfer suite.  Rank 1 is
  best.  The heatmap is descriptive because the suite contains one seed per
  method--dataset pair.}
  \label{fig:cross-dataset-ranks}
\end{figure}

Joint $\pa=1$ ranks first on PTB and Rotten Tomatoes and second on IMDB and
WikiText-2, giving the best average rank, $1.50$
(Fig.~\ref{fig:cross-dataset-ranks}).  Joint $\pa=0.5$ wins IMDB and WikiText-2 but
falls to sixth on Rotten Tomatoes, giving average rank $2.75$.  High control is
again risky: Joint $\pa=2$ ranks eighth on both language-modeling datasets,
although it remains second on Rotten Tomatoes.  The transfer evidence thus
supports a moderate operating region around $\pa\in[0.5,1]$, but it does not
justify a population-level claim until multiple seeds are completed.

\begin{table}[t]
  \centering
  \caption{Validation-selected transfer results for seed 42.  Classification
  reports accuracy (higher is better); language modeling reports perplexity
  (lower is better).}
  \label{tab:transfer-main}
  \scriptsize
  \begin{tabular}{@{}lccccc@{}}
    \toprule
    Method & IMDB (\%) & PTB PPL & RT (\%) & WT2 PPL & Avg. rank \\
    \midrule
    AdamW & 83.78 & 140.08 & 73.73 & 164.91 & 7.25 \\
    Muon & 84.54 & 131.52 & 77.11 & 151.65 & 4.75 \\
    Affine probe $\pa=1$ & 85.08 & \second{129.70} & 76.64 & 149.96 & 3.25 \\
    Joint $\pa=0$ & 84.89 & 130.58 & 75.98 & 150.56 & 5.00 \\
    Joint $\pa=0.5$ & \best{85.88} & 129.79 & 76.17 & \best{149.64} & \second{2.75} \\
    \rowcolor{bestgreen}
    Joint $\pa=1$ & \second{85.41} & \best{129.08} & \best{77.58} & \second{149.65} & \best{1.50} \\
    Joint $\pa=1.5$ & 85.25 & 142.27 & 76.36 & 162.46 & 5.25 \\
    \rowcolor{warningred}
    Joint $\pa=2$ & 84.52 & 156.86 & \second{77.49} & 183.95 & 6.25 \\
    \bottomrule
  \end{tabular}
  \par\vspace{2pt}\footnotesize\raggedright
  Bold and underlining mark the best and second-best values within a column.  These markings are rankings, not statistical significance.
\end{table}

\FloatBarrier

\section{Conclusion}
\label{sec:conclusion}

This paper asked a deliberately limited question: can an optimizer input
control the normalized distribution of an update across prescribed parameter
channels, separately from a global change in update magnitude?  The answer is
yes at a frozen training state.  When one nonzero scalar multiplies the entire
update, normalized allocation is invariant.  In contrast, the coordinate
exponent $p_c$ changes pairwise coordinate log-odds according to an exact
affine law, while the affine exponent $p_\alpha$ changes the augmented
weight--bias geometry through a rank-one positive-semidefinite perturbation
and controls raw affine participation through an exact logistic law.

The spectral analysis also identifies a distinction that is easy to lose in
an implementation-level discussion.  Increasing the scaled bias column can
reduce ideal inverse-square-root gain along the corresponding direction even
when the norm of the transformed physical bias column increases.  Raw input
participation, preconditioner gain, and decoded output are therefore different
objects.  The affine decomposition makes the structural origin of this
direction explicit at a fixed mini-batch, quantifies the remainder introduced
by changing batch means under momentum, and shows that finite polynomial
spectral iterations preserve singular-vector subspaces without requiring them
to equal an exact polar or regularized-inverse map.

The experiments support these mechanism-level statements and delimit their
practical range.  Same-state replay recovers the predicted coordinate
redistribution and affine-logit slope to numerical precision.  Across five
paired seeds on the controlled benchmark, PAdam with $p_c=0.3$ raises mean
test accuracy from $89.54\%$ to $92.09\%$ and minimizes mean test loss, while
joint affine control with $p_\alpha=1$ reaches $92.57\%$ accuracy with low
seed-to-seed variation.  Stronger control is not uniformly better: at
$p_\alpha=2$, training underfits and mean test accuracy falls to $64.09\%$.
The four-task transfer suite favors moderate affine control descriptively, but
its single-seed design does not support uncertainty estimates or a
population-level superiority claim.

Three limitations define the remaining gap.  First, exact reachability is a
one-step result and is not controllability of the stochastic multi-step
training recursion.  Second, the controlled benchmark supplies only five
seeds, and the transfer suite supplies one; additional preregistered,
validation-selected replications are needed.  Third, a common optimization
protocol isolates the allocation controls but does not replace a fully tuned
optimizer comparison.  Future work should develop closed-loop schedules that
target measured allocation coordinates, analyze their stability over entire
trajectories, and test whether the same operating regions persist under
independently tuned learning rates, larger architectures, and broader data
distributions.

The resulting conclusion is mechanistic rather than universal: allocation is
an exactly analyzable optimizer output and can be actuated independently of a
global scale at a fixed state, but the usefulness of a reachable allocation
remains conditional on the task and trajectory.  This separation turns a
qualitative account of implicit bias into a testable control problem while
retaining the assumptions needed for each claim.

\backmatter


\bibliography{references}

\end{document}